\documentclass[runningheads,envcountsame]{llncs}

\usepackage{amsmath,amssymb}
\usepackage{mathtools}
\usepackage{graphicx}
\usepackage{booktabs}
\usepackage{enumitem}
\usepackage{tikz}
\usepackage{pgfplots}
\usepackage{mdframed}
\usepackage[expansion=false]{microtype}
\usepackage[hidelinks]{hyperref}
\pgfplotsset{compat=1.17}

\spnewtheorem{assumption}[theorem]{Assumption}{\bfseries}{\itshape}

\newcommand{\cB}{\mathcal{B}}
\newcommand{\cX}{\mathcal{X}}
\newcommand{\cS}{\mathcal{S}}
\newcommand{\cE}{\mathcal{E}}
\newcommand{\R}{\mathbb{R}}
\newcommand{\E}{\mathbb{E}}
\newcommand{\ind}{\mathbf{1}}
\newcommand{\eps}{\varepsilon}
\newcommand{\dvc}{d_{\mathrm{VC}}}
\newcommand{\dis}{\theta}

\begin{document}

\title{Information Design Against Gaming and Learning Adversaries}
\titlerunning{Information Design Against Gaming and Learning Adversaries}

\author{Madhava Gaikwad}
\authorrunning{M. Gaikwad}
\institute{Independent Researcher, Bengaluru, India\\
\email{gaikwad.madhav@gmail.com}}

\maketitle

\begin{abstract}
A principal who deploys a binary classifier with an abstention option must decide which queries the mechanism abstains on. The right choice depends on the adversary. A gaming adversary already knows the classifier and tries to manipulate features across the boundary, so the principal does best by abstaining on queries close to that boundary. The same boundary-localizing rule is the worst possible choice against a learning adversary who does not know the classifier: each abstention now tells the adversary that the boundary is nearby, which is enough to drive a binary search. We analyze this tension. The two natural defenses, abstaining at a fixed rate and abstaining near the boundary, are Blackwell-incomparable: neither can be simulated by post-processing the other's responses. The number of queries needed to reconstruct the boundary to error $\eps$ is $\tilde\Theta(d/\eps)$ under the first defense and $\Theta(d \log(1/\eps))$ under the second, where $d$ is the VC dimension of the classifier family and $\tilde\Theta$ suppresses factors polylogarithmic in $d$ and $1/\eps$. The first rate is a worst case over query distributions; no reconstruction algorithm can close the gap at the distributions that attain it. We characterize the Pareto frontier between the two defense objectives, and confirm both rates on seven binary-classification tasks spanning tabular, image, and language-model-feature inputs: label-plus-counterfactual access extracts the boundary with up to $200\times$ fewer queries than a published label-only baseline.

\end{abstract}

\section{Introduction}
\label{sec:intro}
A binary classifier deployed behind an API leaks information through its answers. Every answered query gives the queryer one observation about the location of the classifier's \emph{decision boundary}, the surface in feature space that separates the two predicted classes, and an adversary who aggregates many such observations can reconstruct the boundary. This reconstruction problem is known as \emph{model extraction}~\cite{tramer2016stealing}, and it now occurs at industrial scale against deployed systems: in February 2026 Anthropic attributed roughly sixteen million query exchanges against Claude, across some twenty-four thousand fraudulent accounts, to three competing laboratories distilling its capabilities, and responded with behavioral fingerprinting of its API traffic~\cite{anthropic2026distillation}. We call the operator of the classifier the \emph{principal}, following the information-design literature, and we call the queryer the \emph{adversary} or \emph{receiver}. The principal commits in advance to a classifier $B$ on some feature space $\cX$ and to a rule for how the mechanism responds to each query. The response is one of three signals, accept, reject, or abstain. The principal cannot change the rule once an adversary starts querying. A bank's fraud filter, a platform's account-creation gate, and an LLM provider's content classifier are all instances.
The accept and reject signals are pinned down by truthfulness. The abstain signal is the principal's only design lever, and the principal commits to an abstention budget bounding the fraction of queries that may be abstained on. The design problem is choosing \emph{which} queries to abstain on within the budget.

Two adversary types matter in practice. A \emph{gaming} adversary knows $B$ and wants to find a query close to its boundary $\partial B$ that the mechanism still accepts~\cite{hardt2016strategic}. A \emph{learning} adversary does not know $B$ and queries the mechanism in order to reconstruct it, i.e., it mounts a model-extraction attack~\cite{tramer2016stealing}. The principal commits to a single rule that faces both. Our central question is the following.

\begin{mdframed}
\textbf{Question.} \emph{What is the principal-optimal abstention rule against an adversary whose payoff is decreasing in boundary-reconstruction error, subject to an abstention budget? How does it compare to the optimum against a gaming adversary?}
\end{mdframed}

Such mechanisms increasingly run as classification heads on top of frozen language-model encoders, so the boundary $\partial B$ is a hyperplane in embedding space and the principal's instrument is the response rule rather than the classifier itself. The principal then chooses what to disclose alongside the label, ranging from nothing to a \emph{counterfactual}, the nearest input on the opposite side of the boundary, which pins down the boundary's location near the queried point. Section~\ref{sec:experiments} returns to this setting on four such tasks built on a frozen MiniLM-L6 encoder.

\paragraph{A worked example.}
Consider a fraud filter that scores transactions on a single risk feature $x \in [0,1]$ and accepts iff $x < b$ for some unknown threshold $b$. The filter is allowed to abstain (output ``review needed'') on a $\bar\alpha = 0.3$ fraction of queries. The principal commits to a rule $\cE$ that decides, given $x$ and $b$, whether to abstain.

Two natural rules are available.

\emph{Rule A (coarse).} Flip an independent biased coin. Abstain with probability $0.3$ on every query, regardless of $x$ or $B$. When not abstaining, return the truthful label.

\emph{Rule B (boundary-localizing).} Pick $\delta$ so that, under the query distribution $\mu$, a $0.3$ fraction of queries land within distance $\delta$ of $b$. Abstain on exactly those queries. Return the truthful label on the rest.

Both rules have the same average abstention rate $\bar\alpha = 0.3$, but they differ sharply once an adversary watches \emph{which specific queries} get abstained.
Against a gaming adversary who knows $b$ and wants a confirmed accept near $b$, Rule B is the stronger defense. Every query within $\delta$ of $b$ is abstained on, so the adversary cannot get a confirmation of acceptance close to the threshold. Rule A lets $70\%$ of near-boundary queries through.

Against a learning adversary who does not know $b$ and wants to estimate it, Rule B backfires. An abstain response on query $x$ tells the adversary that $b$ lies within $\delta$ of $x$. That is enough to drive a binary search and locate $b$ to within any target precision $\eps$ in $O(\log(1/\eps))$ queries. Under Rule A an adversary who samples its queries from $\mu$ needs on the order of $1/\eps$ of them. One dimension flatters Rule A here: each non-abstained label already reveals on which side of $b$ the query fell, so an adaptive adversary can binary-search on the labels alone and match the $O(\log(1/\eps))$ rate under either rule. The gap that survives adaptivity opens on classes and distributions where labels do not localize the boundary but an abstain that witnesses proximity still does; Theorems~\ref{thm:learner-value-coarse} and~\ref{thm:learner-value-loc} quantify exactly this.
The two rules release different kinds of information. Rule B's abstention encodes a witness of distance to the boundary, where Rule A's, drawn independently of $b$, encodes nothing about it. The rest of the paper makes this precise and shows for which hypothesis classes and query distributions the separation persists.

\paragraph{Main results.}
Let $\cB$ be a hypothesis class of binary classifiers with VC dimension $d$, and let $\mu$ be a fixed distribution on $\cX$: the benchmark for both extraction error and the abstention budget, not a restriction on the adversary's queries (Sect.~\ref{sec:model}). An experiment $\cE$ maps $(x, B) \in \cX \times \cB$ to a signal in $\{\text{accept}, \text{reject}, \text{abstain}\}$ (or a richer signal space, see Sect.~\ref{sec:model}). The principal's \emph{abstention budget} is a cap $\bar\alpha \in [0,1)$ on the probability of an abstain signal under $\mu$; its complement $1 - \bar\alpha$ is the \emph{coverage}. Let $\eps \in (0, 1)$ denote the target precision. We say a receiver \emph{$\eps$-learns} the boundary if its output $\hat B$ satisfies $\Pr_{x \sim \mu}[\hat B(x) \neq B(x)] \leq \eps$ with probability at least $2/3$. Writing $L$ for the number of queries the receiver makes, let $L^*_L(\cE; \eps)$ be the query complexity of $\eps$-learning under $\cE$, and let $L^*_G(\cE)$ be the expected number of queries a gaming adversary needs to confirm a boundary crossing. Throughout, $\tilde\Theta$, $\tilde O$, and $\tilde\Omega$ suppress factors polylogarithmic in $d$ and $1/\eps$; where a $\log(1/\eps)$ appears explicitly it is the leading such factor, and the tilde absorbs the rest. Throughout, $\cE_{\mathrm{coarse}}$ and $\cE_{\mathrm{loc}}^\delta$ denote the multi-dimensional generalizations of Rule~A and Rule~B from the worked example, and $\cE_\lambda = \lambda \cE_{\mathrm{coarse}} + (1-\lambda) \cE_{\mathrm{loc}}^\delta$ their convex combinations, meaning that each query's signal is drawn from $\cE_{\mathrm{coarse}}$ with probability $\lambda$ and from $\cE_{\mathrm{loc}}^\delta$ otherwise, independently across queries.

\begin{enumerate}[leftmargin=*]
\item \textbf{Tight rates at both extremes (Theorems~\ref{thm:learner-value-coarse} and~\ref{thm:learner-value-loc}).}
Write $\dis_{\cB, \mu}(\eps) \in [1, 1/\eps]$ for the active-learning \emph{disagreement coefficient} of $\cB$ under $\mu$~\cite{hanneke2014theory}, defined in Sect.~\ref{sec:model}, which measures the rate at which labels alone shrink the set of plausible boundaries. Then
\[
L^*_L(\cE_{\mathrm{coarse}}; \eps) = \tilde\Theta\!\left( \frac{d \cdot \dis_{\cB, \mu}(\eps) \log(1/\eps)}{1 - \bar\alpha} \right),
\quad
L^*_L(\cE_{\mathrm{loc}}^\delta; \eps) = \Theta\!\left( d \log(1/\eps) \right)\!.
\]
At its worst-case value $\dis_{\cB, \mu}(\eps) = \Theta(1/\eps)$ the coarse rate exceeds the localizing rate by a factor of $1/\eps$; at favorable pairs, such as linear classifiers under log-concave $\mu$, the two rates match up to constants (Theorem~\ref{thm:learner-value-coarse}). The localizing rate requires a mild geometric regularity condition on $\cB$ (Assumption~\ref{ass:regularity}), under which coordinate-wise binary search, driven by the abstention signal, achieves it.

\item \textbf{The two extremes are Blackwell-incomparable (Theorem~\ref{thm:dichotomy}).} Neither experiment can be simulated from the other by post-processing of signals, so no reconstruction algorithm can close the rate gap between them at unfavorable distributions.

\item \textbf{The two extremes defend against different adversaries (Theorem~\ref{thm:divergence}).} $\cE_{\mathrm{loc}}^\delta$ blocks accept signals near the boundary, exactly the signals a gaming adversary needs; $\cE_{\mathrm{coarse}}$ lets them through. But the learning adversary's picture is reversed: at unfavorable distributions, extraction under $\cE_{\mathrm{coarse}}$ costs a factor of $1/\eps$ more queries. Any rule in the family $\cE_\lambda$ between them trades one weakness for the other.

\item \textbf{Pareto frontier (Theorem~\ref{thm:pareto}).} For convex combinations of the two extremes, we characterize the achievable $(L^*_G, L^*_L)$ frontier and bound the welfare cost of robustness to both adversaries.

\end{enumerate}

\section{Related Work}
\label{sec:related}

We study a receiver who draws inferences from responses that the principal controls. This places us at the intersection of three literatures: strategic classification, model extraction, and information design.

\paragraph{Strategic classification and abstention as a defense.}
Beginning with Hardt et al.~\cite{hardt2016strategic}, the strategic-classification literature studies how agents game a known mechanism~\cite{dong2018strategic,chen2020learning,kleinberg2020classifiers}. Most of it assumes the classifier is fixed, and analyzes equilibrium agent behavior; our principal explicitly designs the information environment. Closest to us, Alkarmi et al.~\cite{alkarmi2025abstention} introduce abstention into strategic classification and characterize the optimal abstention rule through a Stackelberg lens, against an adversary who already knows the classifier. We study the complementary problem of an adversary who does not know it. The two papers share the topic of abstention as a defense but are technically orthogonal: their rates govern post-manipulation loss, while ours govern query complexity. Cohen et al.~\cite{cohen2024bayesian} also release truthful but partial information about a deployed classifier against a gaming adversary with a Bayesian prior; our receiver is instead a learning adversary. Reject-option classification~\cite{chow1970optimum,bartlett2008reject,cortes2016rejection} studies abstention from the accuracy perspective rather than the adversary's.

\paragraph{Model extraction and decision-based attacks.}
Model-extraction attacks reconstruct a classifier from queries~\cite{lowd2005adversarial,tramer2016stealing,jagielski2020extraction}, typically without principal commitment. Decision-based attacks~\cite{brendel2018decision,chen2020hopskipjump} reconstruct boundaries from hard-label feedback by binary search along the surface. Our boundary-localizing experiment provides ternary feedback that is strictly more informative than binary, which is what enables the exponential improvement in Theorem~\ref{thm:learner-value-loc}. Connections between extraction and active learning were drawn empirically by~\cite{chandrasekaran2020exploring,pal2020activethief}. In the counterfactual-explanation literature, several works~\cite{aivodji2020counterfactual,dissanayake2024polytope,wang2022dualcf} observe empirically that boundary-proximal explanations accelerate model extraction; Theorem~\ref{thm:learner-value-loc} gives a tight information-theoretic explanation, since such explanations are a boundary-localizing experiment and so fall in the $\Theta(d \log(1/\eps))$ regime. {\sloppy Counterfactuals are one instance of a broader pattern of \emph{explanation leakage}: gradient explanations reconstruct two-layer networks with dimension-independent query counts~\cite{milli2019model}, explanations leak training-set membership~\cite{shokri2021privacy}, and confidence scores admit equation-solving extraction~\cite{tramer2016stealing}.\par}

\paragraph{Active learning and selective classification.}
The disagreement-coefficient framework~\cite{balcan2006agnostic,hanneke2014theory,wang2011smoothness,zhang2014beyond} governs label complexity in active learning, and we use it directly in Theorem~\ref{thm:learner-value-coarse}. Zhu and Nowak~\cite{zhu2022efficient} achieve $\mathrm{polylog}(1/\eps)$ label complexity via learner-chosen abstention; Theorem~\ref{thm:learner-value-loc} achieves the analogous rate via principal-chosen abstention. The connection between active learning and selective classification was anticipated by El-Yaniv and Wiener~\cite{elyaniv2010selective}, and we formalize it in the adversarial setting. On the deployment side, selective classifiers are typically realized as confidence thresholds or trained rejection heads~\cite{geifman2019selectivenet}; our abstain signal is an idealization of such heads, with the threshold set by a budget rather than a target risk. Such heads are themselves attack surfaces: a rejection mechanism robust to adversarially chosen inputs yields a robust classifier at half the radius~\cite{tramer2022detecting}, so rejection is no easier than robust classification; our Theorem~\ref{thm:divergence} concerns instead where the abstain region is placed.

\paragraph{Information design.}
{\sloppy Our framing borrows from Bayesian persuasion~\cite{kamenica2011bayesian} and information design more broadly~\cite{bergemann2019information}. The technical novelty over that literature is the receiver's inferential utility, in place of the action-payoff utilities standard in persuasion, together with the connection to active-learning rates. Cummings et al.~\cite{cummings2015privacy} design mechanisms that elicit data from privacy-aware strategic agents; the differential-privacy literature can be viewed as information design against a worst-case inferential receiver, while we exploit the structure of $\cB$.\par}

\section{Model}
\label{sec:model}

\paragraph{Boundaries and queries.}
Let $\cB$ be a class of binary classifiers $B : \cX \to \{0,1\}$ with $\cX \subseteq \R^p$ for some ambient dimension $p \geq 1$. Write $\dvc(\cB) = d$ for the VC dimension of $\cB$~\cite{vapnik1971uniform}. Our bounds depend on $d$ and not on $p$. Write $\partial B$ for the topological boundary of $B$, $\mathrm{dist}(x, \partial B)$ for the Euclidean distance to it, and $\ind\{\cdot\}$ for the indicator function. Let $\mu$ be a distribution on $\cX$, and $\nu$ a prior on $\cB$. Classifiers are binary throughout; Sect.~\ref{sec:disc} discusses the multi-class extension.

\paragraph{The principal's instrument.}
The principal commits to an \emph{experiment}, a Markov kernel mapping each state, i.e.\ each query-boundary pair $(x, B)$, to a distribution over signals,
\[
\cE : \cX \times \cB \to \Delta(\cS),
\qquad \cS = \{\text{accept}, \text{reject}, \text{abstain}\},
\]
where $\Delta(\cS)$ denotes the set of probability distributions on $\cS$.
We require $\cE$ to be \emph{label-truthful}. The signal ``accept'' has zero conditional probability when $B(x) = 0$, and the signal ``reject'' has zero conditional probability when $B(x) = 1$. The signal ``abstain'' is unconstrained, and is the principal's design lever. The principal operates under an \emph{abstention budget} $\bar\alpha \in [0, 1)$, a cap on the marginal probability of the abstain signal,
\[
\E_{B \sim \nu, x \sim \mu}\bigl[\ind\{\cE(x,B) = \text{abstain}\}\bigr] \;\leq\; \bar\alpha,
\]
and the complementary guarantee $1 - \bar\alpha$ is the \emph{coverage}, the minimum rate of answered queries on traffic drawn from $\mu$. Throughout, $\bar\alpha$ denotes only this budget; both extremal experiments below meet it with equality. A $1/(1-\bar\alpha)$ factor, where one appears, is the price of coverage; the localizing rate carries none (Remark~\ref{rem:delta}). The prior $\nu$ enters only through this budget accounting and through the threshold calibration of Example~\ref{ex:two-experiments}; the learning and gaming guarantees of Sects.~\ref{sec:characterization} and~\ref{sec:divergence} hold for a fixed worst-case $B \in \cB$, so the setting is minimax rather than Bayesian. A principal who deploys a single known classifier may take $\nu$ to be the point mass at it.

\begin{example}[Two extremal experiments]
\label{ex:two-experiments}
We focus on two specific abstention rules.
\begin{itemize}[nosep]
\item \emph{Coarse} $\cE_{\mathrm{coarse}}$. Abstain with probability $\bar\alpha$ independently of $(x, B)$. Otherwise return the truthful label.
\item \emph{Boundary-localizing} $\cE_{\mathrm{loc}}^\delta$. Abstain iff $\mathrm{dist}(x, \partial B) < \delta$, with $\delta$ chosen so that $\Pr_{B \sim \nu, x \sim \mu}[\mathrm{dist}(x,\partial B) < \delta] = \bar\alpha$.
\end{itemize}
Their convex combinations $\cE_\lambda = \lambda \cE_{\mathrm{coarse}} + (1-\lambda) \cE_{\mathrm{loc}}^\delta$ form a one-parameter family with average coverage $1 - \bar\alpha$: on each query, independently of the past, the signal is drawn from $\cE_{\mathrm{coarse}}$ with probability $\lambda$ and from $\cE_{\mathrm{loc}}^\delta$ otherwise.
\end{example}

The boundary-localizing rule of Example~\ref{ex:two-experiments} is one member of a broader class of experiments whose abstain signal acts as a proximity witness. The general form, which our experiments in Sect.~\ref{sec:experiments} also instantiate, is the following.

\begin{definition}[Boundary-localizing experiment, general form]
\label{def:localizing}
An experiment $\cE$ is \emph{boundary-localizing at scale $\delta$} if there is a signal $s^* \in \cS$ that acts as a witness of proximity to the boundary: for some constant $C_1 \in (0, 1]$ and every $B \in \cB$, $x \in \cX$,
\[
\mathrm{dist}(x, \partial B) < \delta \;\Longrightarrow\; \Pr_\cE[\cE(x, B) = s^* \mid x] \;\geq\; C_1,
\]
and the map $x \mapsto \Pr_\cE[\cE(x, B) = s^* \mid x]$ is non-increasing in $\mathrm{dist}(x, \partial B)$.
Because the probability of $s^*$ climbs as the query approaches the boundary, the receiver can use $s^*$ as a distance-to-boundary estimator. Coarse abstention is not boundary-localizing at any scale; $\cE_{\mathrm{loc}}^\delta$ is boundary-localizing at scale $\delta$ with witness signal $s^* = \text{abstain}$.
\end{definition}

\paragraph{Choosing $\delta$ in practice.} The principal knows $B$ but typically lacks a closed form for $\mu$. To meet a budget $\bar\alpha$, draw an unlabeled sample from $\mu$ and set $\delta$ to the $\bar\alpha$-quantile of the empirical distance-to-boundary distribution; a sample of size $\Theta(1/\bar\alpha)$ suffices for the quantile to concentrate, by Glivenko-Cantelli~\cite{vapnik1971uniform}.

\paragraph{Deployment interpretation of abstention.} In our framework the abstain signal is a single bit: the mechanism returns no label. Real systems implement this bit in many ways: a manual-review queue, a request for additional documentation, a retry signal with back-off, or deferral pending human approval. A learning adversary observes the same thing in all four cases --- no label --- so our analysis treats them as equivalent, even though legitimate users experience them differently. The rate $\bar\alpha$ budgets this single bit; our bounds hold for any deployment respecting it.

\paragraph{Receivers.}
A receiver interacts with $\cE$ over $L$ rounds. At each round it chooses a query adaptively from the history of past (query, signal) pairs, and after $L$ rounds it outputs a function of the entire transcript. The receiver knows $\cE$ when choosing its protocol, as is standard in information design. Adaptivity is also the standard threat model in active learning~\cite{hanneke2014theory} and in decision-based attacks~\cite{brendel2018decision,chen2020hopskipjump}, and our lower bounds apply to non-adaptive receivers as well. We consider two receiver types.
\begin{itemize}[nosep]
\item \emph{Learning receiver.} Outputs $\hat B \in \cB$. Its utility is $u_L(\hat B, B) = -\Pr_{x \sim \mu}[\hat B(x) \neq B(x)]$. We say it $\eps$-learns if $u_L \geq -\eps$ with probability at least $2/3$.
\item \emph{Gaming receiver.} Knows $B$ and a manipulation budget $\delta_G > 0$. Seeks any query $x$ with $\mathrm{dist}(x, \partial B) < \delta_G$ and $\cE(x, B) = \text{accept}$. Its utility is the indicator of at least one such accept signal in $L$ queries.
\end{itemize}
The \emph{value} $V_\tau(\cE; L)$ for receiver type $\tau \in \{L, G\}$ is the supremum over receiver protocols of expected utility under $L$ queries. The learning query complexity is $L^*_L(\cE; \eps) = \min\{L : V_L(\cE; L) \geq -\eps\}$. For the gaming receiver we measure the expected number of queries until its first accept signal, $L^*_G(\cE) = \E[\inf\{L : \text{an accept is obtained within the first } L \text{ queries}\}]$. When each query in the manipulation region yields an accept independently with probability $q$, the stopping time is geometric and $L^*_G(\cE) = 1/q$; the value threshold $\min\{L : V_G(\cE; L) \geq 1/2\}$ agrees with this up to a $\Theta(1)$ factor.

\paragraph{What $\mu$ does and does not constrain.}
The distribution $\mu$ enters the model in exactly two places: the learning receiver's error and the abstention budget are both measured under it. It places no restriction on queries: receivers query adaptively anywhere in $\cX$, including regions of vanishing $\mu$-mass, so the abstain frequency they experience may far exceed $\bar\alpha$ --- the binary search behind Theorem~\ref{thm:learner-value-loc} concentrates queries near $\partial B$ without violating a budget that binds only on benign traffic. Our bounds hold for each fixed $\mu$; the $1/\eps$ separation is a worst case over query distributions (Theorem~\ref{thm:learner-value-coarse}).

\paragraph{The disagreement coefficient.}
Our rate at coarse abstention is stated in terms of the active-learning disagreement coefficient~\cite{hanneke2014theory}, written $\dis_{\cB, \mu}(\eps)$ and taking values in $[1, 1/\eps]$. It measures the rate at which the set of plausible boundaries shrinks as the learner acquires labels. Formally,
\[
\dis_{\cB, \mu}(\eps) \;=\; \sup_{B \in \cB} \,\sup_{r \geq \eps}\, \frac{\mu(\mathrm{DIS}_{\cB, \mu}(B, r))}{r},
\]
\[
\mathrm{DIS}_{\cB, \mu}(B, r) \;=\; \bigcup_{B' :\, \Pr_{x \sim \mu}[B(x) \neq B'(x)] \leq r} \{x : B(x) \neq B'(x)\},
\]
where $\mathrm{DIS}_{\cB, \mu}(B, r)$ is the set of points where some boundary within $\mu$-distance $r$ of $B$ disagrees with $B$. The quantity $\dis_{\cB, \mu}(\eps)$ governs the rate at the coarse extreme. There the response rule acts on labels as a random erasure, and the geometry of $\cB$ enters only through how fast the set of boundaries consistent with the observed labels shrinks. At the boundary-localizing extreme the response rule provides side information about distance to $\partial B$, and the relevant complexity is geometric, captured by Assumption~\ref{ass:regularity} below.

\section{Characterization}
\label{sec:characterization}

We first establish the Blackwell-incomparability of the two extremal experiments (Theorem~\ref{thm:dichotomy}), and then state the tight query complexity at each (Theorems~\ref{thm:learner-value-coarse} and~\ref{thm:learner-value-loc}). The two regimes call for different algorithms. Under $\cE_{\mathrm{coarse}}$ the receiver runs active learning under random label erasure. Under $\cE_{\mathrm{loc}}^\delta$ the receiver runs coordinate-wise binary search using the abstention signal as a localization witness.

\subsection{The Two Extremes Are Blackwell-Incomparable}
\label{sec:dichotomy}

\begin{theorem}[Dichotomy of the extremal experiments]
\label{thm:dichotomy}
For any abstention budget $\bar\alpha \in (0, 1/2)$, the two experiments $\cE_{\mathrm{coarse}}$ and $\cE_{\mathrm{loc}}^\delta$ are \emph{Blackwell-incomparable}~\cite{blackwell1953equivalent,lecam1986asymptotic,torgersen1991comparison}. There is no Markov kernel $K_1$ acting on signals (a \emph{garbling}) with $K_1 \circ \cE_{\mathrm{coarse}} = \cE_{\mathrm{loc}}^\delta$ in distribution, and no garbling $K_2$ with $K_2 \circ \cE_{\mathrm{loc}}^\delta = \cE_{\mathrm{coarse}}$.
\end{theorem}

The restriction to $\bar\alpha < 1/2$ reflects deployment, where abstention budgets are small fractions of traffic; at $\bar\alpha = 0$ both experiments coincide with the truthful experiment, so some abstention is needed for the two to differ.
\begin{proof}
We exhibit, in each direction, two specific boundaries $B^{(1)}, B^{(2)} \in \cB$ and a query distribution $\mu$ such that no signal-only Markov kernel $K : \cS \to \Delta(\cS)$ can map one experiment to the other.

\emph{Setup.} Take $\cX = [0,1]$, $\mu$ uniform on $[0,1]$, and let $\cB = \{B_b : b \in [0,1]\}$ be the class of one-dimensional thresholds $B_b(x) = \ind\{x \geq b\}$. This is the simplest class satisfying Assumption~\ref{ass:regularity}. Both experiments use signal space $\cS = \{\text{accept}, \text{reject}, \text{abstain}\}$ and average abstention rate $\bar\alpha$.

\emph{Step 1: $\cE_{\mathrm{loc}}^\delta$ does not garble to $\cE_{\mathrm{coarse}}$.}
Choose witnesses $B^{(1)} = B_{0.5}$ and $B^{(2)} = B_{0.5 + 2\delta}$ with $\delta = \bar\alpha/2$, so that $2\delta < 1/2$, both thresholds lie in $[0,1]$, and abstention has $\mu$-mass $2\delta = \bar\alpha$ under $\cE_{\mathrm{loc}}^\delta$. Under $B^{(1)}$ the abstain region is $x \in (0.5 - \delta, 0.5 + \delta)$; under $B^{(2)}$ it is $x \in (0.5 + \delta, 0.5 + 3\delta)$. These intervals are disjoint.

Suppose for contradiction that $K \circ \cE_{\mathrm{loc}}^\delta = \cE_{\mathrm{coarse}}$ in distribution for a signal-only kernel $K$, so the equality holds for both witnesses simultaneously. Fix a query $x_1 \in (0.5 - \delta, 0.5)$. Under $B^{(1)}$ we have $\cE_{\mathrm{loc}}^\delta(x_1, B^{(1)}) = \text{abstain}$ (since $|x_1 - 0.5| < \delta$), whereas $\cE_{\mathrm{coarse}}(x_1, B^{(1)})$ outputs ``abstain'' with probability $\bar\alpha$ and the truthful label ``reject'' with probability $1 - \bar\alpha$ (the label is ``reject'' because $x_1 < 0.5$). Matching the two distributions at $x_1$ forces $K(\text{abstain}, \text{reject}) = 1 - \bar\alpha$. Now fix a query $x_2 \in (0.5 + 2\delta, 0.5 + 3\delta)$. Under $B^{(2)}$ we have $\cE_{\mathrm{loc}}^\delta(x_2, B^{(2)}) = \text{abstain}$, whereas $\cE_{\mathrm{coarse}}(x_2, B^{(2)})$ outputs ``abstain'' with probability $\bar\alpha$ and the truthful label ``accept'' with probability $1 - \bar\alpha$ (the label is ``accept'' because $x_2 > b^{(2)} = 0.5 + 2\delta$). Matching at $x_2$ forces $K(\text{abstain}, \text{accept}) = 1 - \bar\alpha$. The single row $K(\text{abstain}, \cdot)$ would then have total mass at least $2(1-\bar\alpha) > 1$ for $\bar\alpha < 1/2$, a contradiction.

\emph{Step 2: $\cE_{\mathrm{coarse}}$ does not garble to $\cE_{\mathrm{loc}}^\delta$.}
Take $B^{(1)} = B_{0.25}$ and $B^{(2)} = B_{0.75}$, with $\delta = \bar\alpha/2 < 1/4$ so the $\delta$-neighborhoods $(0.25 - \delta, 0.25 + \delta)$ and $(0.75 - \delta, 0.75 + \delta)$ are disjoint. Under $\cE_{\mathrm{coarse}}$ the signal at a query $x$ depends on the boundary only through the truthful label $B(x)$: it is ``abstain'' with probability $\bar\alpha$ and equals $B(x)$ otherwise. Consequently, for any signal-only kernel $K$, the output distribution $(K \circ \cE_{\mathrm{coarse}})(x, B)$ also depends on $B$ only through $B(x)$. Consider the query $x = 0.25 - \delta/2$. Both witnesses assign it the same label, $B^{(1)}(x) = B^{(2)}(x) = 0$ (since $x < 0.25 < 0.75$), so $(K \circ \cE_{\mathrm{coarse}})(x, B^{(1)})$ and $(K \circ \cE_{\mathrm{coarse}})(x, B^{(2)})$ are identical distributions. But under $\cE_{\mathrm{loc}}^\delta$ the two responses differ: $\cE_{\mathrm{loc}}^\delta(x, B^{(1)}) = \text{abstain}$ (since $|x - 0.25| = \delta/2 < \delta$), whereas $\cE_{\mathrm{loc}}^\delta(x, B^{(2)}) = \text{reject}$ (since $|x - 0.75| > \delta$ and $B^{(2)}(x) = 0$). No signal-only kernel can reproduce $\cE_{\mathrm{loc}}^\delta$ at both $(x, B^{(1)})$ and $(x, B^{(2)})$, because it is forced to produce identical distributions there. \qed
\end{proof}

\subsection{Rate at the Coarse Extreme}
\label{sec:coarse-rate}

The lower bound at the coarse extreme rests on a standard Assouad construction.

\paragraph{The Assouad construction for $\cB$.}
\label{app:assouad-construction}
For boundary class $\cB$ with $\dvc(\cB) = d$, fix a shattered set $\{z_1, \ldots, z_d\} \subset \cX$. Such a set exists by definition of VC dimension. For each sign vector $\sigma \in \{-1, +1\}^d$, let $B_\sigma \in \cB$ denote a boundary with $B_\sigma(z_j) = (1 + \sigma_j)/2$. The shattering property ensures that all $2^d$ such boundaries lie in $\cB$. We refer to $\sigma_j$ as the $j$-th $B$-coordinate. The Assouad query distribution $\mu_{\mathrm{ass}}$ places mass $\Theta(\eps/d)$ on each $z_j$, with the \emph{total} mass on the shattered points equal to a large enough constant multiple of $\eps$ (strictly exceeding the $\eps$ error budget), and the remaining mass on a region where all $B_\sigma$ agree. Each misclassified $z_j$ contributes $\Theta(\eps/d)$ to the $\mu_{\mathrm{ass}}$-error, so to keep the error within the $\eps$ budget an $\eps$-learner must classify all but a small constant fraction of the $z_j$ correctly. (Taking the total mass to be exactly $\eps$, as a mass of $\eps/d$ per point would, makes the budget vacuous, since misclassifying every $z_j$ then still incurs error only $\eps$; a constant factor above $\eps$ makes the constraint binding and does not change the final rate.)

The lower bound uses the sharp form of Assouad's lemma.

\begin{lemma}[Sharp Assouad lemma~{\cite{yu1997assouad}}]
\label{lem:assouad-sharp}
Let $\Sigma = \{-1, +1\}^d$. For each $\sigma \in \Sigma$ let $P_\sigma$ be a distribution over observations, and define Hamming loss $\ell(\hat\sigma, \sigma) = (1/d) \sum_j \ind\{\hat\sigma_j \neq \sigma_j\}$. Writing $H^2(P, Q) = \tfrac{1}{2} \int (\sqrt{dP} - \sqrt{dQ})^2$ for the squared Hellinger distance, suppose for all adjacent pairs $(\sigma, \sigma')$ differing in coordinate $j$ that $H^2(P_\sigma, P_{\sigma'}) \leq \gamma_j$. Then for any estimator $\hat\sigma$,
\[
\max_{\sigma \in \Sigma} \E[\ell(\hat\sigma, \sigma)] \;\geq\; \frac{1}{2d} \sum_{j=1}^d \bigl(1 - \sqrt{\gamma_j}\bigr).
\]
\end{lemma}

\begin{theorem}[Learner's value under coarse abstention]
\label{thm:learner-value-coarse}
For any boundary class $\cB$ with $\dvc(\cB) = d$, query distribution $\mu$, and abstention budget $\bar\alpha \in [0, 1)$,
\[
L^*_L(\cE_{\mathrm{coarse}}; \eps) = \tilde\Theta\!\left( \frac{d \cdot \dis_{\cB,\mu}(\eps) \log(1/\eps)}{1-\bar\alpha} \right).
\]
In the worst case $\dis_{\cB,\mu}(\eps) = \Theta(1/\eps)$, and the bound becomes $\tilde\Theta(d \log(1/\eps)/((1-\bar\alpha)\eps))$. For favorable classes such as linear classifiers under log-concave $\mu$~\cite{wang2011smoothness}, $\dis_{\cB,\mu}(\eps) = O(1)$ and the bound becomes $\Theta(d \log(1/\eps)/(1-\bar\alpha))$, matching the localizing rate up to constants.
\end{theorem}

\begin{proof}
\emph{Upper bound.} Under $\cE_{\mathrm{coarse}}$, each query produces ``abstain'' with probability $\bar\alpha$ independent of everything, and otherwise the truthful label. The receiver discards abstain rounds. The remaining $L(1-\bar\alpha)$ effective rounds are noiseless labeled samples on adaptively chosen queries, which is the realizable active-learning setting~\cite{balcan2006agnostic,hanneke2014theory}. By Hanneke~\cite[Theorem 5.1]{hanneke2014theory}, the disagreement-based active learner CAL (Cohn--Atlas--Ladner) achieves $\eps$-learning with effective query count
\[
\tilde O\!\left( \dis_{\cB,\mu}(\eps) \cdot d \cdot \log(1/\eps) \right).
\]
Dividing by $1-\bar\alpha$ gives the total-query upper bound.

\emph{Lower bound.} We use the Assouad construction above, in which each shattered point $z_j$ carries mass $\Theta(\eps/d)$ and the total mass on $\{z_j\}$ is a fixed constant multiple $c\eps$ of $\eps$ with $c$ large enough (say $c > 8$). For adjacent sign vectors $\sigma, \sigma'$ differing in coordinate $j$ only, the response distributions of $\cE_{\mathrm{coarse}}$ on $L$ adaptive queries differ only on the event ``$x_t = z_j$ \emph{and} non-abstain''. The probability of this event is $O((1-\bar\alpha)\eps/d)$ per round, and the bound holds even when $x_t$ is chosen adaptively, since by Bayes the marginal probability of $x_t = z_j$ over $\mu_{\mathrm{ass}}$-induced strategies is $O(\eps/d)$. Tensorizing the per-round Hellinger contributions gives $H^2(P_\sigma^{(L)}, P_{\sigma'}^{(L)}) = O(L(1-\bar\alpha)\eps/d)$. By the sharp Assouad lemma (Lemma~\ref{lem:assouad-sharp}), whenever $L = O(d/((1-\bar\alpha)\eps))$ with a small enough constant the maximum expected Hamming error over $\sigma$ is at least $1/4$, i.e.\ a constant fraction of the coordinates $z_j$ are misclassified. Each misclassified $z_j$ contributes its mass $\Theta(\eps/d)$ to the $\mu_{\mathrm{ass}}$-error, so misclassifying a $1/4$-fraction incurs error $\Omega(c\eps) > \eps$ for $c > 8$, exceeding the budget. Hence $\eps$-learning forces $L \geq \Omega(d/((1-\bar\alpha)\eps))$ for the sign-vector parameter. (The constant $c$ enters only the lower-order constants and does not change the rate.)

The matching $\log(1/\eps)$ factor and the $\dis_{\cB,\mu}(\eps)$ dependence follow from the realizable lower bound of Hanneke~\cite[Theorem 4.1]{hanneke2014theory}, which lower-bounds the label complexity of $\eps$-learning by the logarithm of the $\eps$-covering number of $\cB$ under $\mu$; for a VC class on adversarially-chosen $\mu$ this logarithm is $\Omega(d \log(1/\eps))$. Combined with the Assouad bound on adaptive Hellinger and the disagreement-coefficient construction of Beygelzimer et al.~\cite{beygelzimer2009importance} for $\dis = \Omega(1/\eps)$, the worst case over $\cB$ gives $L \geq \tilde\Omega(d \log(1/\eps) / ((1-\bar\alpha)\eps))$. Dividing by $1-\bar\alpha$ recovers the total-query lower bound. \qed
\end{proof}

\subsection{Rate at the Boundary-Localizing Extreme}
\label{sec:loc-rate}

The localizing rate requires a mild geometric regularity condition.

\begin{assumption}[Boundary regularity]
\label{ass:regularity}
\sloppy There exist a shattered set $\{z_1, \ldots, z_d\}$ and directions $\{v_1, \ldots, v_d\}$ such that for each $j$ and each $B \in \cB$, the map $t \mapsto B(z_j + t v_j)$ is monotone, with a crossing point that varies continuously in $B$.
\end{assumption}

Linear classifiers, axis-aligned boxes and rectangles, polynomial threshold functions of bounded degree, and decision trees of bounded depth all satisfy this assumption~\cite{hanneke2014theory}.

\begin{theorem}[Learner's value under localizing abstention]
\label{thm:learner-value-loc}
Suppose $\cB$ satisfies Assumption~\ref{ass:regularity} and the threshold satisfies $\delta = O(\eps)$. Then
\[
L^*_L(\cE_{\mathrm{loc}}^\delta; \eps) = \Theta\!\left(d \log(1/\eps)\right),
\]
independently of the abstention budget $\bar\alpha$ (Remark~\ref{rem:delta}).
\end{theorem}

\begin{proof}
\emph{Upper bound: coordinate-wise binary search.} Fix the shattered set $\{z_j\}_{j=1}^d$ and monotone directions $\{v_j\}_{j=1}^d$ of Assumption~\ref{ass:regularity}; to $\eps$-learn $B$ it suffices to localize the crossing along each $v_j$ to within $\eps$.

For each $j$, maintain an interval $I_j = [a_j, b_j]$ containing the crossing along $v_j$, initially of length $O(1)$, and query its midpoint $x = z_j + ((a_j+b_j)/2) v_j$. A label response halves $I_j$ by monotonicity; an abstain response ends the coordinate outright, certifying that the crossing lies within $\delta$ of the midpoint, a window of width $2\delta \leq \eps$ (choosing the constant in $\delta = O(\eps)$ accordingly). Every response is certified progress: abstention is deterministic given the query --- far queries always draw labels, near ones the witness --- so nothing is erased, no repetition is needed, and the budget, binding on $\mu$-traffic rather than the attack stream (Sect.~\ref{sec:model}), never taxes the receiver. After at most $\lceil \log_2(1/\eps) \rceil$ queries per coordinate, $|I_j| \leq \eps$, for $O(d \log(1/\eps))$ queries in total.

\emph{Lower bound: Fano on ternary signals.} Consider axis-aligned thresholds. For $w \in \{0, \eps, 2\eps, \ldots, 1\}^d$ let $B_w(x) = \ind\{\forall j: x_j > w_j\}$. There are $\lfloor 1/\eps \rfloor^d$ such boundaries, and this class satisfies Assumption~\ref{ass:regularity}. Each query response under $\cE_{\mathrm{loc}}^\delta$ takes one of three values, accept, reject, or abstain, so each query carries at most $\log_2 3$ bits. Identifying $w$ requires $\log(1/\eps^d) = d \log(1/\eps)$ bits. By Fano's inequality~\cite{cover2006elements}, any estimator with error probability below $1/2$ requires $I(w; \text{responses}) \geq \tfrac{1}{2}\log_2 M - 1$, where $M = \lfloor 1/\eps\rfloor^d$ so $\log_2 M = d\log_2\lfloor 1/\eps\rfloor$. Since the $L$ ternary responses carry $I(w; \text{responses}) \leq L \log_2 3$ bits, we obtain $L \geq (d \log_2\lfloor 1/\eps\rfloor - 2)/(2\log_2 3) = \Omega(d \log(1/\eps))$. No coverage factor enters: abstains are signal symbols already counted in the $\log_2 3$, not erasures. \qed
\end{proof}

\begin{remark}[Role of $\delta$]
\label{rem:delta}
The budget reaches the learner only through the calibrated width $\delta(\bar\alpha,\mu,\nu)$, and $\delta = O(\eps)$ is what makes abstention sharpen rather than cap resolution. If $\delta \gg \eps$, every query in the surviving interval is abstained on and an abstain certifies proximity only at scale $\delta$, so the window argument degrades to error $\max(\eps,\delta)$ in $\Theta(d\log(1/\max(\eps,\delta)))$ queries. The exact rule could in principle escape this floor by binary-searching its sharp label-to-abstain edge, which sits at known offset $\delta$ from $\partial B$; this knife-edge property fails for every other member of Definition~\ref{def:localizing} and for approximate $\delta$, so we do not rely on it. A generous budget, $\delta(\bar\alpha) \gg \eps$, thus robustly floors the learner's resolution at $\delta(\bar\alpha)$; the theorem's regime is $\bar\alpha$ tuned so that $\delta(\bar\alpha) = O(\eps)$.
\end{remark}

\section{Divergence and the Pareto Frontier}
\label{sec:divergence}

We now compare the two extremal experiments against the gaming and learning receivers.

\begin{theorem}[Divergence at the two extremes]
\label{thm:divergence}
Suppose $\cB$ satisfies Assumption~\ref{ass:regularity} and the principal operates under abstention budget $\bar\alpha \in [0,1)$. Then:
\begin{enumerate}[nosep]
\item \emph{Gaming gap.} If $\delta \geq \delta_G$, then $V_G(\cE_{\mathrm{loc}}^\delta; L) = 0$ for any finite $L$. Under $\cE_{\mathrm{coarse}}$, $V_G(\cE_{\mathrm{coarse}}; L) \to 1$ as $L \to \infty$.
\item \emph{Learning gap.} At worst-case $\dis_{\cB,\mu}(\eps) = \Theta(1/\eps)$, the learner needs a factor of $1/\eps$ more queries to extract the boundary under $\cE_{\mathrm{coarse}}$ than under $\cE_{\mathrm{loc}}^\delta$. By Theorems~\ref{thm:learner-value-coarse} and~\ref{thm:learner-value-loc} the ratio is
\[
\frac{L^*_L(\cE_{\mathrm{coarse}}; \eps)}{L^*_L(\cE_{\mathrm{loc}}^\delta; \eps)} \;=\; \tilde\Theta\!\left(\frac{1}{\eps}\right),
\]
which is unbounded as $\eps \to 0$.
\end{enumerate}
\end{theorem}

\begin{proof}
\emph{Part 1 (Gaming gap).} A gaming receiver succeeds iff at least one query $x$ with $\mathrm{dist}(x, \partial B) < \delta_G$ produces signal ``accept''. Under $\cE_{\mathrm{loc}}^\delta$ the abstain region is exactly $\{x : \mathrm{dist}(x, \partial B) < \delta\}$. If the principal sets $\delta \geq \delta_G$, every query inside the gaming receiver's manipulation region is abstained on, so no ``accept'' signal arrives there. Hence $V_G(\cE_{\mathrm{loc}}^\delta; L) = 0$ for all $L$. Under $\cE_{\mathrm{coarse}}$ the abstention is independent of position. A query $x$ in the manipulation region with $B(x) = 1$ is accepted with probability $1 - \bar\alpha$ at each draw. After $L$ such queries the success probability is $1 - \bar\alpha^L$, which tends to $1$ as $L \to \infty$ for any $\bar\alpha < 1$.

\emph{Part 2 (Learning rate gap).} Immediate from Theorems~\ref{thm:learner-value-coarse} and~\ref{thm:learner-value-loc}: at $\dis_{\cB,\mu}(\eps) = \Theta(1/\eps)$ the two rates differ by the factor $\dis_{\cB,\mu}(\eps)/(1-\bar\alpha)$, i.e.\ $\tilde\Theta(1/\eps)$ for fixed $\bar\alpha$. \qed
\end{proof}

\begin{theorem}[Pareto frontier of the convex-combination family]
\label{thm:pareto}
Suppose $\cB$ satisfies Assumption~\ref{ass:regularity} and the localizing threshold satisfies $\delta \geq \delta_G$. The family $\cE_\lambda$, $\lambda \in (0,1]$, of Example~\ref{ex:two-experiments} consists of budget-$\bar\alpha$ experiments with
\[
L^*_G(\cE_\lambda) = \frac{1}{\lambda(1-\bar\alpha)},
\]
\[
L^*_L(\cE_\lambda; \eps) = \tilde\Theta\!\left( \min\!\left( \frac{d \cdot \dis_{\cB,\mu}(\eps) \log(1/\eps)}{1-\bar\alpha},\; \frac{d \log(1/\eps)}{(1 - \lambda)(1-\bar\alpha)} \right) \right).
\]
Both bounds reduce to Theorem~\ref{thm:learner-value-coarse} at $\lambda = 1$; as $\lambda \to 0$ the second branch exceeds the rate of Theorem~\ref{thm:learner-value-loc} by a $1/(1-\bar\alpha)$ factor explained in the proof. The achievable set $\{(L^*_G(\cE_\lambda), L^*_L(\cE_\lambda; \eps)) : \lambda \in (0,1]\}$ is closed under time-sharing, and no $\cE_{\lambda'}$ in the family dominates $\cE_\lambda$ on both objectives unless $\lambda' = \lambda$.
\end{theorem}

\begin{proof}
\emph{Convexity under time-sharing.} The family $\{\cE_\lambda\}_{\lambda \in (0,1]}$ is closed under randomization between two values of $\lambda$: a $t$-mixture of $\cE_{\lambda_1}$ and $\cE_{\lambda_2}$ produces $\cE_{t \lambda_1 + (1-t)\lambda_2}$, since both are convex combinations of the two extremes. The achievable values are linear in this $\lambda$-randomization in the relevant ranges, giving convexity of the achievable set within the family.

{\sloppy \emph{$L^*_G$ formula.} The gaming receiver queries near-boundary points $x$ with $\mathrm{dist}(x, \partial B) < \delta_G \leq \delta$. Under $\cE_\lambda$ such a query draws the coarse signal with probability $\lambda$, which returns ``accept'' with probability $1-\bar\alpha$ since $B(x) = 1$, and draws the localizing signal with probability $1-\lambda$, which abstains because the point lies within $\delta$. The per-query accept probability is therefore $\lambda(1-\bar\alpha)$, the stopping time is geometric, and $L^*_G(\cE_\lambda) = 1/(\lambda(1-\bar\alpha))$.\par}

\emph{$L^*_L$ formula.} The learner takes the better of two strategies. The first discards every abstain and runs the active learner of Theorem~\ref{thm:learner-value-coarse} on the remaining labels. Treating an abstain as an erasure only discards information, so the count it produces is a valid upper bound; the resulting label rate is at least the coverage $1-\bar\alpha$, and the analysis of Theorem~\ref{thm:learner-value-coarse} then gives $\tilde O\!\left(d \cdot \dis_{\cB,\mu}(\eps)\log(1/\eps)/(1-\bar\alpha)\right)$ queries, independent of $\lambda$. This is the first branch. The second runs the coordinate-wise binary search of Theorem~\ref{thm:learner-value-loc}, with one complication: under $\cE_\lambda$ an abstain is ambiguous, since the coarse component also abstains, with probability $\bar\alpha$ independent of position. The receiver repeats each midpoint until either a label arrives --- probability at least $\lambda(1-\bar\alpha)$ per draw, even inside the tube --- certifying a halving, or a run of abstains that a far point would produce with negligible probability, certifying proximity. Repetitions of order $1/((1-\lambda)(1-\bar\alpha))$ per step suffice, up to logarithmic factors absorbed by $\tilde\Theta$, giving $O\!\left(d\log(1/\eps)/((1-\lambda)(1-\bar\alpha))\right)$, the second branch; the factor prices disambiguation within the mixture, and the pure rule, having no coin abstains, carries none. The learner picks the smaller of the two. At $\lambda = 1$ the second branch diverges and the first remains, recovering Theorem~\ref{thm:learner-value-coarse}; as $\lambda \to 0$ the second branch reduces to Theorem~\ref{thm:learner-value-loc} and, since $\dis_{\cB,\mu}(\eps) \geq 1$, attains the minimum.

\emph{Within-family non-domination.} $L^*_G(\cE_\lambda) = 1/(\lambda(1-\bar\alpha))$ is strictly decreasing in $\lambda$. The first branch of $L^*_L(\cE_\lambda; \eps)$ is constant in $\lambda$ and the second is increasing in $\lambda$, so their minimum is non-decreasing in $\lambda$. Hence no $\cE_{\lambda'}$ in the family weakly improves on $\cE_\lambda$ in both objectives unless $\lambda' = \lambda$. \qed
\end{proof}

The frontier has two regimes. For small $\lambda$ the localizing component governs the learner's second strategy, $L^*_G$ blows up, and $L^*_L$ stays at $O(d \log(1/\eps)/(1-\bar\alpha))$. For large $\lambda$ the localizing component is rare, $L^*_G$ stays at $O(1/(1-\bar\alpha))$, and $L^*_L$ grows until the first strategy takes over. The crossover happens when $1 - \lambda \asymp 1/\dis_{\cB,\mu}(\eps)$: at worst-case $\dis_{\cB,\mu}(\eps) = \Theta(1/\eps)$ this is $\lambda \approx 1 - \eps$, and there both branches reach $\tilde\Theta(d \log(1/\eps)/\eps)$, recovering the coarse rate of Theorem~\ref{thm:learner-value-coarse}; for favorable classes with $\dis_{\cB,\mu}(\eps) = O(1)$ the crossover sits at $\lambda$ bounded away from $1$ and both branches reach $\Theta(d \log(1/\eps))$. We do not claim Pareto-optimality of $\cE_\lambda$ over arbitrary label-truthful budget-$\bar\alpha$ experiments outside the family.

\section{Experiments}
\label{sec:experiments}

We test the rate predictions of Theorems~\ref{thm:learner-value-coarse} and~\ref{thm:learner-value-loc} on seven binary-classification tasks\footnote{Code and full experimental details: \url{https://github.com/krimler/market-design2}.} spanning three modalities: tabular features (ACSIncome~\cite{ding2021folktables}, $d=68$), raw image pixels (MNIST 1 vs 7~\cite{lecun1998gradient} and Fashion-MNIST T-shirt vs Pullover~\cite{xiao2017fashion}, $d=784$), and text represented by embeddings, fixed vector representations computed by a frozen MiniLM-L6 sentence encoder~\cite{wang2020minilm} (SST-2 sentiment~\cite{socher2013recursive}, AG News Sports vs rest~\cite{zhang2015character}, Civil Comments toxicity~\cite{borkan2019nuanced}, TREC question-type NUM vs rest~\cite{li2002learning}, all $d=384$). Throughout this section $d$ denotes the input dimension of the (possibly embedded) feature vector. For linear classifiers in $\R^d$ this matches the VC dimension up to an additive one for the intercept. The principal (the \emph{victim} model) is an unregularized logistic regression on each task. The adversary trains a logistic-regression \emph{surrogate} on responses to $L$ random queries from a held-out pool. We measure \emph{fidelity} $F$, the fraction of held-out points on which surrogate and victim agree, and report the extraction error $1 - F$. We compare three extraction strategies:
\begin{itemize}[nosep]
\item \emph{Protocol A (random label-only).} Random queries, victim returns labels only. This is $\cE_{\mathrm{coarse}}$ with the abstain channel turned off.
\item \emph{ActiveThief.} The published label-only extraction baseline of Pal et al.~\cite{pal2020activethief}, which queries actively to maximize surrogate uncertainty. Same labels-only access as Protocol A, but with strategically chosen queries.
\item \emph{Protocol B (label + counterfactual).} Random queries, victim returns labels and the closest opposite-label point on its boundary. This is a boundary-localizing experiment in the sense of Definition~\ref{def:localizing}, though strictly more informative than the ternary $\cE_{\mathrm{loc}}^\delta$ of Theorem~\ref{thm:learner-value-loc}.
\end{itemize}

The counterfactual on each query is the closed-form orthogonal projection onto the principal's boundary, plus an infinitesimal step to the opposite side. Each $(L, \text{strategy})$ cell is averaged over $20$ seeds.

\begin{figure}[t]
\centering
\includegraphics[width=0.85\textwidth]{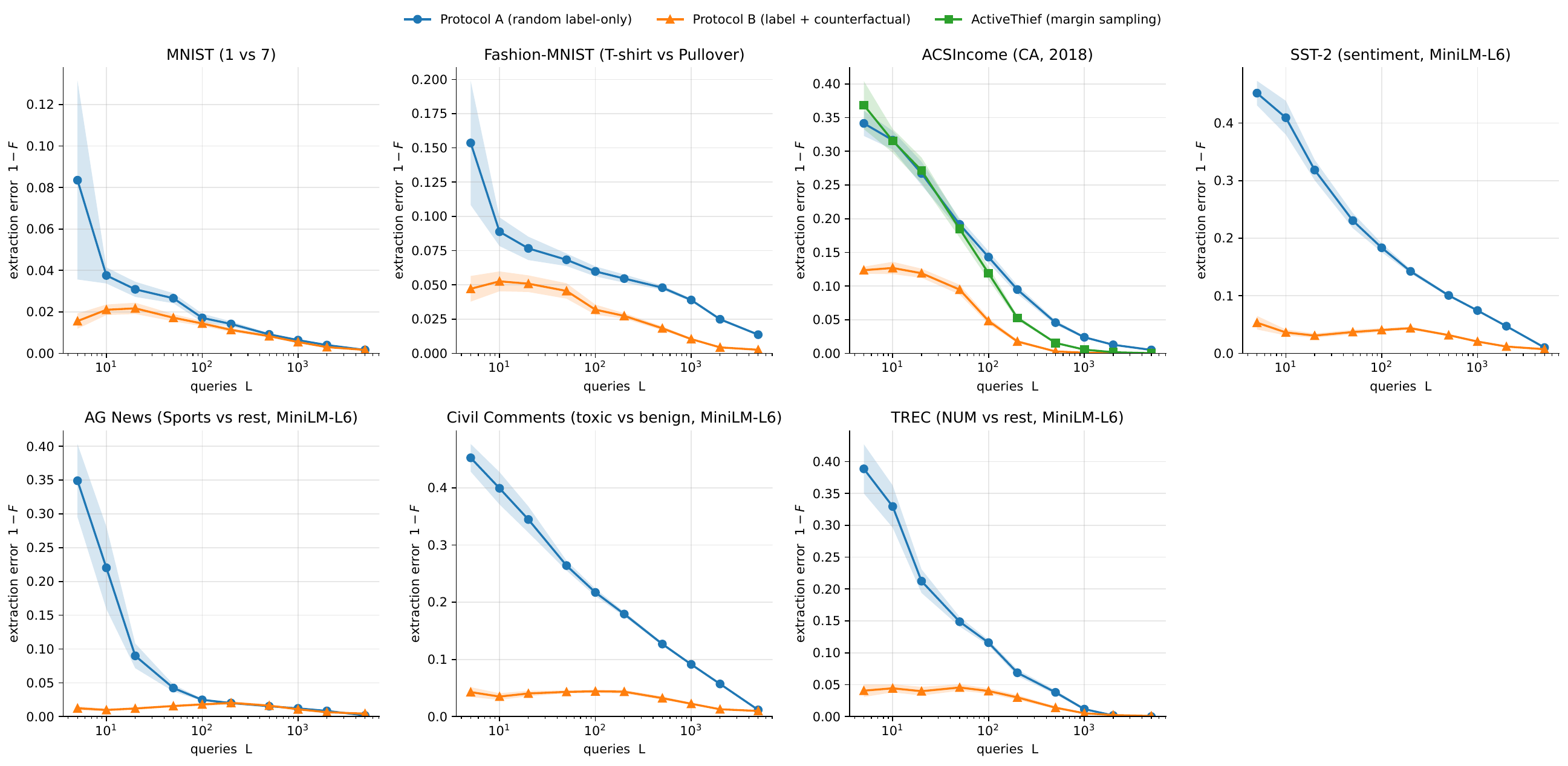}
\caption{Extraction error vs query budget on all seven datasets. Protocol A is random label-only access; Protocol B is label-plus-counterfactual access; ActiveThief appears on the ACSIncome panel only (Table~\ref{tab:activethief}). Shaded bands are $95\%$ confidence intervals over $20$ seeds. The gap is largest for $L \lesssim d$, where Protocol B has localized the boundary. ACSIncome ($d=68$) shows the predicted exponential cliff for Protocol B between $L=50$ and $L=200$; at $d=784$ (MNIST, Fashion-MNIST) no comparable cliff appears, though the exponential decay shape persists.}
\label{fig:fidelity}
\end{figure}

\paragraph{Cross-modal rate verification.} For each (dataset, protocol) we fit the predicted functional form on the empirical decay curve (Fig.~\ref{fig:fidelity}). Protocol A is fit to a power law $\log(1-F) = c - \rho \log L$. Logistic regression on i.i.d.\ realizable labels converges in 0-1 loss at rate $L^{-1/2}$ due to the surrogate-loss gap, so we expect $\rho \approx 1/2$ on benign distributions and slower on harder ones. Protocol B is fit to an exponential $\log(1-F) = c - \beta L$, with $\beta > 0$ predicted from Theorem~\ref{thm:learner-value-loc}. Table~\ref{tab:rates} reports the fits.

\begin{table}[t]
\centering
\caption{Fitted decay rates and end-to-end query gap at $\eps = 0.02$: power-law exponent $\rho$ for Protocol A, exponential rate $\beta$ (with fit $R^2$) for Protocol B.}
\label{tab:rates}
\setlength{\tabcolsep}{3pt}
\scriptsize
\begin{tabular}{lccccccc}
\toprule
Dataset & $d$ & Victim acc. & $\rho$ (A) & $\beta$ (B) & $R^2$ (B) & Gap at $\eps{=}0.02$ \\
\midrule
ACSIncome (CA, 2018)        & $68$  & $0.786$ & $0.80$ & $5.1 \times 10^{-3}$ & $0.92$ & ${>}25\times$ \\
TREC question type          & $384$ & $0.887$ & $1.12$ & $1.6 \times 10^{-3}$ & $0.95$ & $40\times$ \\
SST-2 sentiment             & $384$ & $0.831$ & $0.60$ & $3.7 \times 10^{-4}$ & $0.86$ & $200\times$ \\
Civil Comments toxic        & $384$ & $0.775$ & $0.60$ & $3.2 \times 10^{-4}$ & $0.84$ & $200\times$ \\
AG News Sports              & $384$ & $0.981$ & $0.58$ & $2.6 \times 10^{-4}$ & $0.72$ & ${\sim}5\times$ \\
MNIST 1 vs 7                & $784$ & $0.995$ & $0.56$ & $5.0 \times 10^{-4}$ & $0.83$ & ${\sim}5\times$ \\
Fashion T-shirt vs Pullover & $784$ & $0.941$ & $0.32$ & $6.1 \times 10^{-4}$ & $0.81$ & ${\sim}10\times$ \\
\bottomrule
\end{tabular}
\end{table}

Six of seven datasets fit Protocol A's exponent in $\rho \in [0.32, 0.80]$, clustered around the predicted $0.5$ from logistic regression's surrogate-loss convergence. The seventh (TREC) has $\rho = 1.12$, faster than the surrogate-loss prediction and approaching the $\rho = 1$ rate of optimal ERM on a realizable task. Protocol B's exponential fit is positive on all seven datasets with $R^2 \geq 0.72$, confirming the boundary-localizing prediction of Theorem~\ref{thm:learner-value-loc}. The end-to-end gap at $\eps = 0.02$ ranges from $5\times$ on the easy-to-extract topic and image tasks to $200\times$ on the harder text-classification tasks. The experiments validate the qualitative separation between coarse and boundary-localizing access; the precise Protocol A exponent is set by logistic regression rather than the active-learning lower bound.

\paragraph{Comparison to a published label-only baseline.} On ACSIncome we compare random label-only access (Protocol A) against ActiveThief~\cite{pal2020activethief}, the published active-learning extraction baseline, against label+counterfactual access (Protocol B). Table~\ref{tab:activethief} reports extraction error at five query budgets.

\begin{table}[t]
\centering
\caption{Extraction error $1 - F$ on ACSIncome under three access protocols, by query budget $L$.}
\label{tab:activethief}
\begin{tabular}{lccccc}
\toprule
$L$ & $5$ & $50$ & $200$ & $1000$ & $5000$ \\
\midrule
Protocol A (random)          & $0.342$ & $0.192$ & $0.095$ & $0.024$  & $0.005$ \\
ActiveThief                  & $0.369$ & $0.185$ & $0.053$ & $0.006$  & $0.0002$ \\
Protocol B (counterfactual)  & $0.124$ & $0.095$ & $0.018$ & $0.001$  & $0.00006$ \\
\bottomrule
\end{tabular}
\end{table}

ActiveThief is roughly $2\times$ faster than random sampling, with empirical exponent $\rho = 1.43$ versus the random baseline's $\rho = 0.80$. ActiveThief's exponent exceeds $1$, faster than the optimal active-learning rate suggests, reflecting finite-sample acceleration on this specific 68-dim task rather than the asymptotic rate. Even so, Protocol B remains $3$--$6\times$ ahead of ActiveThief in the practically relevant regime $L \in [200, 1000]$. The qualitative separation predicted by Theorem~\ref{thm:learner-value-loc} thus holds against the published optimized label-only attack as well as against random sampling.

\paragraph{Gaming-resistance verification (Theorem~\ref{thm:divergence} Part 1).} On ACSIncome and TREC, we simulate a gaming attacker that knows the victim's boundary, restricts queries to a held-out pool within $\delta_G$ of it ($\delta_G$ set to the $90$th percentile of pool distance-to-boundary), and submits up to $50$ queries per trial. Across $200$ trials and three response protocols, the attacker's success rate is $1.000$ under no abstention, $1.000$ under coarse $30\%$ abstention, and $0.000$ under boundary-localizing abstention with $\delta = \delta_G$, on both datasets. The attacker queries only victim-accepting points, so under coarse abstention a single non-abstained query confirms acceptance; localizing abstention covers every such query and blocks the attack entirely, matching Theorem~\ref{thm:divergence} Part 1.

\paragraph{Shape of the frontier (Theorem~\ref{thm:pareto}).} Figure~\ref{fig:synthetic} plots the two formulas of Theorem~\ref{thm:pareto} as $\lambda$ varies between the extremes, with the disagreement coefficient at its worst case $1/\eps$. The crossover of the proof is visible: the binary-search branch governs everywhere except the window $1 - \lambda \lesssim 1/\dis_{\cB,\mu}(\eps)$, inside which the active-learning branch takes over and $L^*_L$ climbs to the coarse rate of Theorem~\ref{thm:learner-value-coarse}. The figure is an evaluation of the bound rather than a simulation, so the endpoint magnitudes are illustrative.

\begin{figure}[t]
\centering
\begin{minipage}[t]{0.5\textwidth}
\centering
\begin{tikzpicture}
\begin{axis}[
    xlabel={Gaming complexity $L^*_G(\cE_\lambda)$},
    label style={font=\footnotesize}, tick label style={font=\scriptsize},
    ylabel={Learning complexity $L^*_L(\cE_\lambda; \eps)$},
    xmode=log, ymode=log,
    xmin=1, xmax=1000,
    ymin=10, ymax=5000,
    width=\textwidth, height=0.58\textwidth,
    grid=both, grid style={gray!20},
    legend pos=north east,
    legend style={font=\scriptsize},
    every axis plot/.append style={thick}
]
\addplot[blue, mark=*, mark size=2pt] coordinates {
    (1.43, 1974) (1.44, 1974) (1.45, 1316) (1.46, 987)
    (1.50, 395) (1.59, 197) (1.79, 99) (2.04, 66)
    (2.86, 39.5) (4.76, 28.2) (7.14, 24.7) (14.3, 21.9)
    (28.6, 20.8) (71.4, 20.1) (143, 19.9) (286, 19.8)
    (714, 19.8) (1000, 19.8)
};
\addlegendentry{Pareto frontier $\cE_\lambda$}
\addplot[red, mark=square*, only marks, mark size=4pt] coordinates {(1.43, 1974)};
\addlegendentry{$\cE_{\mathrm{coarse}}$ ($\lambda{=}1$)}
\addplot[green!60!black, mark=triangle*, only marks, mark size=4pt] coordinates {(1000, 19.7)};
\addlegendentry{$\cE_{\mathrm{loc}}^\delta$ ($\lambda{\to}0$)}
\end{axis}
\end{tikzpicture}
\end{minipage}
\caption{Pareto frontier of Theorem~\ref{thm:pareto} at $d=3$, $\bar\alpha=0.3$, $\eps=0.01$, evaluated at the worst-case disagreement coefficient $\dis_{\cB,\mu}(\eps) = 1/\eps$ with all constants set to one; each point is one value of $\lambda \in (0,1]$, both axes log-scale. The $\lambda \to 0$ limit has $L^*_G = \infty$ and is plotted at the axis edge.}
\label{fig:synthetic}
\end{figure}
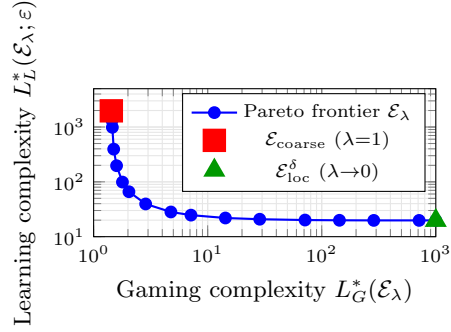

\paragraph{Scope of empirical validation.} Protocol B is strictly more informative than the ternary $\cE_{\mathrm{loc}}^\delta$ of Theorem~\ref{thm:learner-value-loc}, since counterfactual access returns a continuous opposite-side point rather than a $\{$accept, reject, abstain$\}$ signal. Both belong to the boundary-localizing class of Definition~\ref{def:localizing}, and our experiments confirm that the class admits exponential extraction rates; isolating $\cE_{\mathrm{loc}}^\delta$ specifically would require a victim that responds to near-boundary queries with abstention in place of a counterfactual, which we leave as a refinement.

\section{Discussion}
\label{sec:disc}

The dichotomy between coarse and boundary-localizing abstention is, we think, the part of this paper most likely to survive contact with richer models: on hard distributions, no single static rule serves both adversaries. What we do not yet know falls into five groups.

\paragraph{Closing the rate for non-regular boundary classes.} Theorem~\ref{thm:learner-value-loc} requires Assumption~\ref{ass:regularity}. The assumption holds for linear classifiers, bounded-degree polynomial thresholds, and decision trees of bounded depth. It fails for RBF networks and unions of unboundedly many intervals. For such classes the binary-search algorithm no longer recovers the boundary in $O(\log(1/\eps))$ queries per coordinate, and the Fano-on-ternary-signals lower bound also breaks. We do not know the right rate to replace $\Theta(d \log(1/\eps))$; our guess is $\Theta(d \log^\gamma(1/\eps))$ or $\Theta(d/\eps^\gamma)$ for some $\gamma \in (0, 1)$ set by a smoothness parameter of $\cB$. On synthetic RBF boundaries, which may violate regularity, we observe the localizing rate degrading to roughly $d/\sqrt{\eps}$, in line with this guess at $\gamma = 1/2$.

\paragraph{Multi-class extensions.} Everything above is binary; with $K > 2$ classes each ingredient changes. The boundary becomes a union of up to $\binom{K}{2}$ pairwise surfaces; the relevant capacity is the Natarajan dimension~\cite{natarajan1989learning} or the DS dimension characterizing multi-class learnability~\cite{brukhim2022characterization}; active learning extends with subtler label-complexity behavior~\cite{krishnamurthy2017active}; and consistent multi-class abstention already differs from Chow's binary rule~\cite{ramaswamy2018consistent}. The two extremes are affected asymmetrically: the Fano counting improves only logarithmically (each response carries $\log_2(K{+}1)$ bits), while binary search degrades structurally: a path may cross several pairwise surfaces, Assumption~\ref{ass:regularity} must hold per class pair, and an abstain witnesses proximity to \emph{some} surface without naming which. Whether the gap survives, and which multi-class dimension governs both extremes, we do not know.

\paragraph{Non-linear principal objectives.} Our principal minimizes adversary value subject to an abstention budget, which keeps the principal's problem linear in $\cE$. A real deployer also weighs accuracy on legitimate users against adversary cost, and may face manual-review costs non-linear in the abstain rate. Under such objectives the optimum need not be either extreme. The standard concavification arguments of Bayesian persuasion~\cite{kamenica2011bayesian} do not directly apply, because the receiver's utility is inferential rather than action-payoff; the persuasion analog for inferential receivers we leave unresolved.

\paragraph{Dynamic experiments.} Theorem~\ref{thm:divergence} separates the two extremal \emph{static} rules by a factor polynomial in $1/\eps$, and the combinations $\cE_\lambda$ inherit the gap. A time-varying rule that adapts to the query history might collapse it; resetting the rule's state at fixed intervals is one concrete candidate. The persuasion analog is multi-period information design~\cite{ely2017beeps}, which sometimes circumvents single-period impossibility results. Whether a dynamic rule can achieve both at constant-factor cost would settle whether the dichotomy of Theorem~\ref{thm:dichotomy} is an artifact of static rules; we have evidence in neither direction.

\paragraph{Broader empirical scope.} Our experiments fix the victim to logistic regression; bounded-depth decision trees and low-degree polynomial thresholds also satisfy Assumption~\ref{ass:regularity} but need counterfactual constructions beyond the closed-form orthogonal projection we use. A comparison against decision-based attacks such as HopSkipJumpAttack~\cite{chen2020hopskipjump} would situate the rate against the strongest published label-only attack, not only ActiveThief~\cite{pal2020activethief}. Finally, our four language-feature experiments assume the attacker queries the same MiniLM-L6 encoder as the victim; an attacker with a different encoder is learning the victim's classifier composed with a transformation between the encoders, and neither the disagreement coefficient nor the binary-search localization obviously transfers.

\begin{credits}
\subsubsection{\discintname}
The author has no competing interests to declare that are relevant to the content of this article.
\end{credits}

\bibliographystyle{splncs04}
\bibliography{refs}

\end{document}